\documentclass{article}

\usepackage{arxiv}

\usepackage[utf8]{inputenc} % allow utf-8 input
\usepackage[T1]{fontenc}    % use 8-bit T1 fonts
\usepackage{hyperref}       % hyperlinks
\usepackage{url}            % simple URL typesetting
\usepackage{booktabs}       % professional-quality tables
\usepackage{amsfonts}       % blackboard math symbols
\usepackage{nicefrac}       % compact symbols for 1/2, etc.
\usepackage{microtype}      % microtypography
\usepackage{xcolor}         % colors
\usepackage{graphicx} % Required for inserting images
\usepackage{amsmath,amsthm,amssymb,amscd}
\usepackage{mathtools}
\usepackage{subcaption,comment}
\usepackage{wrapfig}
\usepackage{algorithm}
\usepackage{algpseudocode}
\usepackage{caption}
\usepackage{natbib}
\newcommand\argmin{{\mathrm{argmin}}}

\newcommand\R{{\mathbb R}}

\newcommand{\Ex}{\mathbb{E}}
\newcommand{\de}{\,\text{d}}

\newcommand\newsubcap[1]{\phantomcaption%
       \caption*{\figurename~\thefigure(\thesubfigure): #1}}
\newtheorem{theorem}{Theorem}[section]

\newtheorem{remark}{Remark}[section]

\title{Nonlinear denoising score matching for enhanced learning of structured distributions}

\author{Jeremiah Birrell\\
    Department of Mathematics\\
 Texas State University\\
 San Marcos, TX 78666, USA \\ 
  \texttt{jbirrell@txstate.edu} \\
   \And
    Markos A. Katsoulakis\\
    Department of Mathematics and Statistics\\
  University of Massachusetts Amherst\\
  Amherst, MA 01003, USA \\
  \texttt{markos@umass.edu} \\
\And
    Luc Rey-Bellet\\
    Department of Mathematics and Statistics\\
  University of Massachusetts Amherst\\
  Amherst, MA 01003,  USA \\
  \texttt{luc@umass.edu} \\
\And
Benjamin J. Zhang\\
Division of Applied Mathematics\\ 
Brown University\\
Providence, RI 02912, USA\\
\texttt{benjamin\_zhang@brown.edu}\\
  \And
    Wei Zhu\\
    School of Mathematics\\
 Georgia Institute of Technology\\
Atlanta, GA 30332,  USA \\
\texttt{weizhu@gatech.edu}
}

\begin{document}
\maketitle

%% Abstract
\begin{abstract}
%% Text of abstract
We present a novel method for training score-based generative models which uses nonlinear noising dynamics to improve learning of structured distributions.  Generalizing to a nonlinear drift allows for additional structure to be incorporated into the dynamics, thus making the training better adapted to the data,  e.g., in the case of multimodality or (approximate) symmetries.  Such structure can be obtained from the data by an inexpensive preprocessing step. The nonlinear dynamics introduces new challenges into training which we address in two ways: 1) we develop a new nonlinear denoising score matching (NDSM) method, 2) we introduce neural control variates in order to reduce the variance of the NDSM training objective. We demonstrate the effectiveness of this method on several examples: a) a collection of low-dimensional examples, motivated by clustering in latent space, b) high-dimensional images, addressing issues with mode imbalance, small training sets, and approximate symmetries, the latter being a challenge for methods based on equivariant neural networks, which require exact symmetries, c) latent space representation of high-dimensional data, demonstrating  improved performance with greatly reduced computational cost. Our method learns score-based generative models with less data by flexibly incorporating structure arising in the dataset. 
\end{abstract}


\section{Introduction}
Generative modeling is a rapidly evolving collection of techniques for learning high-dimensional probability distributions using samples  \cite{ruthotto2021introduction}.
Likelihood-free or simulation-based inference \cite{cranmer2020frontier} and surrogate models \cite{ZHU201956} provide  important and growing sets of applications of generative models \cite{dax2021real,chen2021mo,chen2022inverse,kastner2023gan,baptista2024bayesian,teng2024generative}. Since the introduction of score-based generative models (SGMs) \citep{song2020score,ho2020denoising} %and denoising diffusion models \citep{ho2020denoising},
there has been intense interest in accelerating their training and generation processes. While being able to perform conditional sampling \cite{batzolis2021conditional} and produce high quality samples with minimal mode collapse \citep{xiao2021tackling},  SGM's generation process is comparably more expensive than generative adversarial networks (GANs) \citep{goodfellow2014generative} or normalizing flows \cite{grathwohl2018ffjord}.  In this paper, we introduce and develop the use of \emph{nonlinear} forward processes in SGMs. While the original formulation of SGMs by \citet{song2020score} included the possibility of nonlinear forward processes, their practical use has not been well-explored due to the great empirical success of linear forward processes. Here we find that, when designed wisely, nonlinear diffusion processes provide a way to incorporate structure about the target distribution, such as multimodality or \emph{approximate} symmetry, into the SGM to produce higher quality samples with less data.  

The class of forward processes we consider is inspired by overdamped Langevin dynamics (OLD). SGMs that use the Ornstein-Uhlenbeck process, the OLD of the normal distribution, as the forward process have been widely explored. In this work, we use more general OLD corresponding to \emph{Gaussian mixture models} (GMM). These GMMs can be learned cheaply via a preprocessing step on a subset of (unlabeled or sparsely labeled) data, and then used to define the drift for the forward process. The use of GMMs is well-motivated for SGMs. Gold standard implementations of generative models \citep{vahdat2021score,rombach2022high} perform SGMs in the latent space. As distributions in the latent space are often multimodal in scientific applications \citep{ding2019deciphering,li2023searching}, we argue GMMs are a natural choice for nonlinear forward processes. Furthermore, nonlinear drift terms based on OLDs corresponds with choosing reference measures other than the Gaussian. Via the \emph{optimal control} interpretation of SGM \citep{berner2022optimal,zhang2023mean} we can also interpret the nonlinear drift term  as arising from a state cost in the control problem.

Denoising score-matching (DSM) \citep{vincent2011connection,song2020score} is the most frequently used objective function for learning the score function as it avoids computing derivatives of the score. Their use, however, relies on knowing the probability transition kernel of the forward process, which is only possible for linear processes. Therefore, we develop a nonlinear version of DSM (NDSM) in Section \ref{sec:NDSM}. The main idea is that even for nonlinear drift functions, the \emph{local} transition probability functions are \emph{approximately} normal. Implementing NDSM is a challenge in itself. Our method is related to the local-DSM approach of \cite{singhals}. A key innovation in our analysis is introduction of a new variance-reduced  NDSM-loss in Theorem \ref{thm:NDSM}; this is achieved through identifying and canceling a   high-variance mean-zero term. We further build on this by proposing a novel \emph{neural control variates} method to produce low variance estimates of the objective function in Theorem \ref{thm:NDSM_CV}. Neural control variates are a deep learning version of classical control variates \citep{asmussen2007stochastic}, for variance reduction, and may be of independent interest elsewhere.

Numerical experiments on MNIST and its approximate $C_2$-symmetric variant validate our theoretical arguments. In particular, in Table~\ref{tab:result_mnist}, we see that the inception score (IS) and Fr\'{e}chet inception distance (FID) with our model vastly outperform the standard OU with DSM SGM. Moreover, we also show our model can learn generative models with substantially \emph{less data}, as shown in Figure \ref{fig:intro_a}. Our method is particularly well suited to latent space representations of high dimensional data that exhibit clustering, as they are a natural fit for the Gaussian mixture preprocessing step and the reduced dimensionality lowers the cost associated with simulating the nonlinear SDE, leading to GM+NDSM-CV also outperforming OU+DSM  on the basis of performance versus training time. We demonstrate this by the example in Section \ref{sec:MNIST_latent}, as previewed in Figure \ref{fig:intro_b}.

\begin{figure}
\begin{subfigure}{0.56\textwidth}
    \centering
    \includegraphics[width = .48\textwidth,trim = {45 45 45 45},clip]{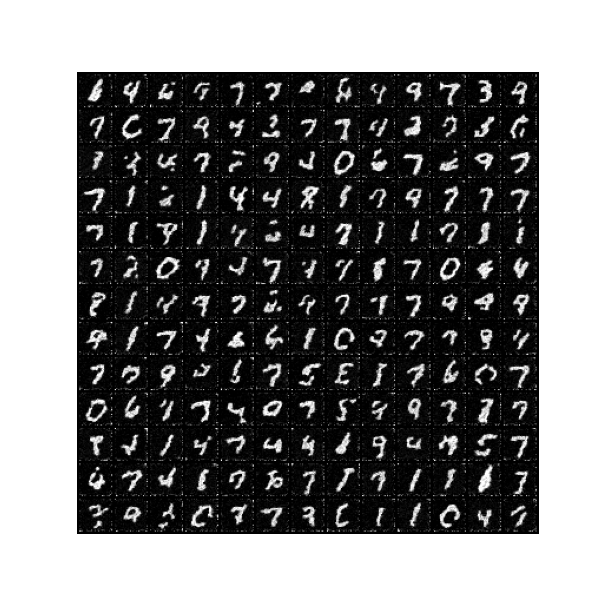}
    \includegraphics[width = .48\textwidth,trim = {45 45 45 45},clip]{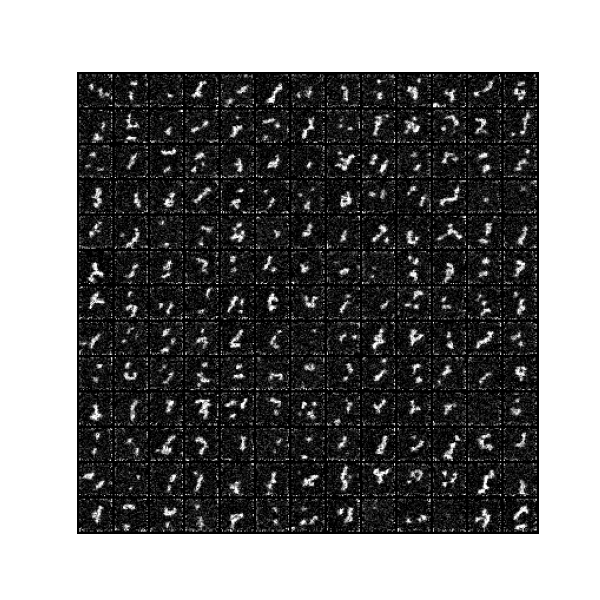}
    \subcaption{$N=6000$: NDSM  (left),   DSM (right)}\label{fig:intro_a}
    \end{subfigure}
%     \begin{subfigure}{0.32\textwidth}
%     \includegraphics[width = \textwidth,trim = {45 45 45 45},clip]{figures/low_data_pics/nlsm_10000.png}
%     \subcaption{GM+NDSM-CV, N = 10000}
%     \end{subfigure}
%\begin{subfigure}{0.24\textwidth}
 %   \centering
  %  \includegraphics[width = \textwidth,trim = {45 45 45 45},clip]{figures/low_data_pics/ou_dsm_6000.png}
   % \subcaption{OU+DSM, N = 6000}
    %\end{subfigure}
    % \begin{subfigure}{0.33\textwidth}
    % \includegraphics[width = \textwidth,trim = {45 45 45 45},clip]{figures/low_data_pics/ou_dsm_10000.png}
    % \subcaption{OU+DSM, N = 10000}
    % \end{subfigure}
       % \begin{subfigure}{0.48\textwidth}
    %\includegraphics[width = .48\textwidth,trim = {45 45 45 45},clip]{figures/low_data_pics/nlsm_14000.png}
     %\includegraphics[width = .48\textwidth,trim = {45 45 45 45},clip]{figures/low_data_pics/ou_dsm_14000.png}
     %\subcaption{$N = 14000$: NDSM (left), DSM (right)}
%\end{subfigure} 
        \begin{subfigure}{0.40\textwidth}
     \includegraphics[width = .98\textwidth]{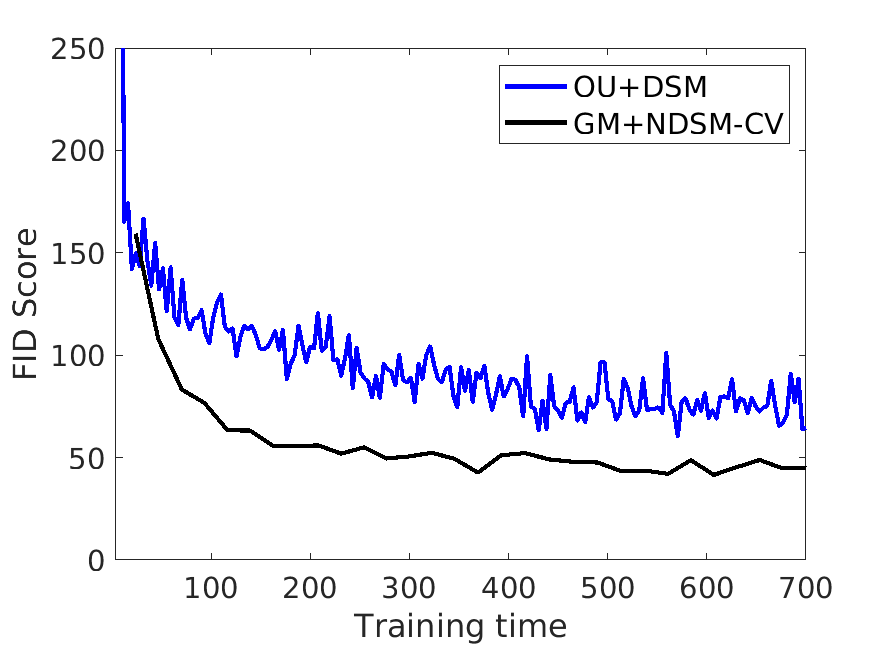}
     \subcaption{FID Score versus training time.}\label{fig:intro_b}
\end{subfigure} 
\caption{(a) MNIST in the low data regime ($N$ training samples), comparing OU+DSM with our new GM+NDSM-CV method. GM+NDSM-CV can learn well with less data. (b) While GM+NDSM-CV is more computationally expensive, it can still outperform OU+DSM in term of performance versus training time. }
\end{figure}

\subsection{Contributions}
\begin{itemize}
    \item We introduce score-based generative modeling with nonlinear forward noising processes as a method for enhanced learning of distributions with structure. Specifically, an inexpensive preprocessing step is applied to (a subset of) the data to construct a Gaussian mixture (GM) reference measure. This reference measure constitutes the initial distribution of the denoising process and also determines the nonlinear drift of the forward  noising process. The GM nonlinear drift, being informed by the structure of the data, then leads to improved performance. 
    \item NDSM, a nonlinear version of the denoising score-matching objective function, is introduced to facilitate the practical implementation SGMs with nonlinear drift term. A robust NDSM implementation in both the sample and latent spaces relies on introducing a varianced-reduced NDSM-loss and a novel neural control variates method.
    \item Numerical experiments validate our claims of improved performance, including better FID and inceptions scores as well as the ability to learn from fewer training samples and  a reduction in mode imbalance; the latter is especially important if the trained model is to be used in downstream tasks, such  computing statistics  or as a surrogate model in a stochastic optimization problem.
\end{itemize}

\subsection{Related work} 
Incorporating structure to accelerate learning of generative models has been extensively studied for a variety of generative algorithms. In particular, \emph{equivariant generative models} such as structure-preserving GANs \cite{birrell2022structure}, equivariant normalizing flows \cite{kohler2020equivariant,garcia2021n}, and equivariant SGMs \cite{lu2024structure} and diffusion models \cite{hoogeboom2022equivariant}, have been empirically shown to be beneficial for learning distributions with \textit{exact} group symmetry. These successes have been supported by theoretical analyses \cite{chen2023sample,chen2023statistical}. Moreover, structure in a broader sense, such as mathematical structure, for designing and training generative models has been more frequently used for enabling faster training with limited data \cite{gu2024lipschitzregularizedgradientflowsgenerative,
zhang2024wasserstein,gu2024combining}. 

In contrast to equivariant score-based and diffusion models, the nonlinear noising dynamics and GMM prior proposed in the present work is less restrictive in its assumptions and is therefore able to handle distributions with \emph{approximate symmetries},  by which we mean that the action of a group on the data distribution $\pi$ produces distributions close (but not equal) to $\pi$, according to some notion of closeness between distributions. We also acknowledge recent work \cite{singhals},  which proposes a local-DSM for SGMs with nonlinear noising dynamics. Our work contributes additional methodological improvements to manage the increased variance encountered when estimating the NDSM objective function, which is particularly crucial when the timestep in the forward and backward SDEs is small; we show that incorporating this variance reducing term results in a substantial increase in the performance of the method. Furthermore, the nonlinear process we use is learned from the training data through a cost-effective preprocessing step. See Section 6 for further details.

\section{Score-based generative models with nonlinear noising dynamics} %Nonlinear noising dynamics encourages structure learning

Let $\pi$ be the target data distribution on $\R^d$ known only through a finite set of samples $\{y_i\}_{i  = 1}^N$. Score-based generative modeling considers a pair of diffusion processes whose evolution are time reversals of each other. Given a drift vector field $f: \R^d \times [0,T] \to \R^d$ and  a diffusion coefficient $\tilde\sigma: [0,T] \to \R$, these diffusion processes $Y(s)$ and $X(t)$ are defined on the time interval $s,t \in [0,T]$ by
\begin{align}\label{eq:forward_backward_SDEs}
    \begin{dcases}
        dY(s) = -f(Y(s),T-s) ds + \tilde\sigma(T-s) dW(s),\,Y(0) \sim \pi \\
        dX(t) = (f(X(t),t) + \tilde\sigma(t)^2 \nabla \log \eta(X(t),T-t ) dt + \tilde\sigma(t) dW(t), \,X(0) \sim \rho_0,
    \end{dcases}
\end{align}
where $Y(s) \sim \eta(\cdot,s)$ and $\rho_0$ is some initial reference measure. If $\rho_0 = \eta(\cdot,T)$, then $X(t)$ exactly follows the time reversed evolution of $Y(s)$, i.e., $X(t) \sim \eta(\cdot,T-t)$. Typically, $f$ and $\tilde\sigma$ are chosen such that $Y(s)$ evolves to a Gaussian as quickly as possible. Linear SDEs such as the Ornstein-Uhlenbeck process (OU) ($f(x,t) = x/2$, $\tilde\sigma(t) = 1$), ``variance exploding'' processes ($f(x,t) = 0$, with $\tilde\sigma(t)$ growing quickly in $t$), or the critically damped Langevin process \cite{dockhorn2021score} (where the diffusion coefficient $\sigma(t)$ is in matrix form), are the most common choices in practice. Moreover, the score function is learned via a score-matching (SM) objective, typically the denoising SM.

% While score-based generative models produce high quality images, they are relative slow at generating images, as compared with GANs \cite{xiao2021tackling}. For this reason, reducing the simulation time, i.e., making $T$ small is of particular interesting. 

We choose nonlinear noising dynamics arising from the overdamped Langevin dynamics \cite{pavliotis2014stochastic}. If the noising process corresponded with a Langevin process with stationary distribution $\rho_0 \propto \exp(-V(x))$, then it would correspond to a drift term $f = -\nabla V(x)$ and $\tilde\sigma(t) = \sqrt{2}$. In this paper, we choose $\rho_0$ to be a  Gaussian mixture model (GMM), and choose the nonlinear noising dynamics accordingly. However, we emphasize that the NDSM method we develop can be applied to any nonlinear drift.

% Doing so would have three main advantages. Moreover,Lastly, the learned score-function would likely have complexity and would not have to learn in regions of the state space that has low probability with respect to $\pi$. We argue this has a training advantage. 

\subsection{Gaussian mixture forward dynamics}\label{sec:GM_noising}

We use the freedom to choose the drift, $f$, in \eqref{eq:forward_backward_SDEs} to match the noising dynamics to the structure of the data.  Specifically, we fit a GMM with weights $w_i$, means $\mu_i$, and covariances $\Sigma_i$ to (a subset of) the data as a preprocessing step; we note that the  required preprocessing does not require labeled data.  More specifically, our implementation uses   the \texttt{mixture.GaussianMixture} and \texttt{mixture.BayesianGaussianMixture} methods from scikit-learn \cite{scikit-learn}; the former uses a fixed number of modes, as specified by the user, while the latter infers the appropriate number of modes from the data. We emphasize that we do not pre-specify the mode weights, means, or covariances; they are fit from a subset of the data. That subset is chosen to be small enough that this preprocessing step makes up a negligible portion of the overall computational cost in all of our tests.

The density
\begin{align}\label{eq:eta_star}
 &\eta_*(y)=\sum_{i=1}^K w_i N_{\mu_i,\Sigma_i}(y) \,\,\,\,\text{where } \,\,\,\, N_{\mu_i,\Sigma_i}(y)\coloneqq(2\pi)^{-d/2}\det(\Sigma_i)^{-1/2} \exp\left(-\frac{(y- \mu_i)\cdot \Sigma_i^{-1}(y-\mu_i)}{2}\right) 
\end{align}
is then invariant  under the GM noising dynamics
\begin{align}\label{eq:GM_nosing} 
    dY(s) = -\nabla V(Y(s))ds + \sqrt{2}dW(s), \quad V(y) \coloneqq -\log \left[\sum_{i=1}^K w_i N_{\mu_i,\Sigma_i}(y)\right].
\end{align}
The  noising dynamics \eqref{eq:GM_nosing}, together with the corresponding denoising dynamics and initial distribution $\rho_0=\eta_*$,  encodes important aspects of the structure of the data, such as multimodality and (approximate) symmetries. We demonstrate that this leads to improved performance, especially when using a small training set, and also helps to prevent mode imbalance.  Pseudocode for the corresponding noising and denoising dynamics, using the Euler-Maruyama (EM) discretization for a given choice of timesteps $\Delta t_n$, $n=0,...,n_f-1$ can be found in Algorithms  \ref{alg:forward_noising} and \ref{alg:denoising} respectively.  
\begin{algorithm}
\caption{Forward Noising Dynamics (given the data distribution $\pi$)}\label{alg:forward_noising}
\begin{algorithmic}[1]
\State Sample $y_1,...,y_B$ from $\pi$
\State $y_{j,0}=y_j$
\For{$n=0,..., n_f-1$}
    \State Sample $Z_{n+1}\sim N(0,I)$
    \State $y_{i,n+1}=y_{i,n}-\nabla V(y_{i,n})\Delta t_n+\sqrt{2\Delta t_n}Z_{n+1}$
\EndFor
\end{algorithmic}
\end{algorithm}

\begin{algorithm}
\caption{Denoising Dynamics (given the trained score model $s_\theta$)}\label{alg:denoising}
\begin{algorithmic}[1]
\State Sample $x_1,...,x_B$ from $\eta_*$ \,\,\,(see Eq.~\ref{eq:eta_star})
\State $x_{j,0}=x_j$
\For{$n=0,..., n_f-1$}
    \State Sample $Z_{n+1}\sim N(0,I)$
    \State $x_{i,n+1}=x_{i,n}+(\nabla V(x_{i,n})+2s_\theta(x_{i,n},T-t_n))\Delta t_n+\sqrt{2\Delta t_n}Z_{n+1}$
\EndFor
\end{algorithmic}
\end{algorithm}
The noising dynamics start in the data distribution, $\pi$, while the denoising dynamics start in the Gaussian mixture prior $\eta_*$ \eqref{eq:eta_star}, which is the invariant distribution for the noising dynamics. The parameters for $\eta_*$ are obtained by fitting a Gaussian mixture model to some fraction of the training data as described above.

\paragraph{Structure-preserving properties} 
 The reference $\rho_0$ is typically chosen so that the noising process respects certain aspects of the structure of the true distribution; this provides an alternate way to impose structure in SGMs. Approaches based on equivariant neural networks work well when the target distribution has exact symmetries. But imposing this type of rigid structure \emph{a priori}, such as in \cite{lu2024structure,hoogeboom2022equivariant} may excessively constraint the model when the distribution only exhibits approximate symmetry.

\paragraph{Faster convergence to (quasi-)stationary distribution}
The GMM $\rho_0=\eta_*$, \eqref{eq:eta_star}, is learned from samples of $\pi$, so we anticipate $\pi$ to be closer to $\rho_0$ in the Kullback-Leibler divergence or total variance distance than the normal distribution used by DSM with linear dynamics, and thus $\pi$ will converge to $\rho_0$ more quickly under the forward noising dynamics. More specifically, we note that  reaching stationarity in multimodal distributions may be quite slow, thus it is more reasonable to assume that the forward process converges to a \emph{quasi-stationary} distribution of the process quickly \cite{lelievre2022quasi,collet2013quasi}. This quasi-stationary distribution can still effective serve as the reference measure, especially when combined with an appropriate weighting of the modes as in our GMM fitting procedure.  In short, we do not need samples to escape ``their mode'' of the GMM and explore the others in order to get good performance, we only need a good approximation to stationarity within each mode, together with appropriate mode weights. 

% If $\pi$ were closer in KL or TV to $\rho_0$ than to a normal distribution, then while the rate of convergence would be the same, the time to effectively get to the stationary distribution is reduced. Langevin dynamics based on GMMs converge to stationarity slowly, but with the right initial condition, will converge to a good quasistationary distribution quickly. A quasistationary distribution is a metastable state in the space of probability distributions that the noising process stays near for a long time. 

\paragraph{Optimal control interpretation}
   The (mean-field) control formulation of score-based generative models \cite{berner2022optimal,zhang2023mean,zhang2024wasserstein} provides further insight and interpretations of our method. SGMs are optimizers of the control problem
    \begin{align}
        &\min_{v,\rho} \left\{- \hspace{-5pt}\int_{\R^d} \rho(x,T) \log \pi(x) \de x+ \int_0^T \hspace{-5pt}\int_{\R^d} \left(\frac{1}{2}|v(x,t)|^2 - \nabla \cdot f(x,t) \right) \rho(x,t) \de x \de t\right\} \\
        &\text{s.t. }\, \partial_t \rho + \nabla \cdot( (f + \tilde\sigma v)\rho) = \frac{\tilde\sigma^2}{2} \Delta \rho,\, \rho(x,0) = \rho_0(x)\,. \nonumber
    \end{align}
    Note that in practice, one can view the minimization  as being  over $v$, with $\rho$  determined by an SDE whose drift depends on $v$.    Here, $\nabla \cdot f$ can be interpreted as a \emph{state cost}, i.e., the cost incurred by being at a particular location in the state space. For typical choices of $f$ (linear functions), this state cost is zero or constant in space, meaning that no region of space is preferred over any other. When $f$ is a nonlinear function of space, then $\nabla\cdot f$ discourages the solution to visit regions of high $\nabla \cdot f$ value. In particular, our choice of $f$ is based on Langevin dynamics, meaning that the state cost is of the form $\nabla\cdot f(x,t) = - \Delta V(x)$, the Laplacian of a potential function. Areas of strict convexity are penalized more than areas where $V$ is concave or where $\Delta V$ is small. These geometric interpretations may provide insight in designing $f$ in future investigations.

   % These geometric interpretations that may be of interest for future investigation.\MK{Areas of strict convexity are less important than areas where $V$ is concave or where the $\Delta V$ is low. Maybe we could get advantage of these features in future work in designing $f$, see also \cite{zhang2024wasserstein}.}

    % \MK{what does the latter mean  in the case of \eqref{eq:GM_nosing}?} \BZ{In the context of our paper, $\nabla\cdot f = \Delta V$, so there might be some geometric interpretation. }

% \textcolor{olive}{Main desiderata: (1) reduce simulation time (i.e., make $T$ small), (2) preserve and exploit structure to learn distributions faster, (3) learn score only in relevant parts of the state space. One way to achieve these goals is to choose a nonlinear noising process. }

% \textcolor{olive}{Choosing a nonlinear noising process based on overdamped Langevin dynamics is analogous to choosing different reference measure for SGM.}

% \textcolor{olive}{Main contributions: (1): introduced SGM with nonlinear noising processes for learning distributions with structure; intuitive and inexpensive preprocessing step informs how to design nonlinear process, (2) a nonlinear version of denoising score matching ; (3) Robust implementation of NDSM through new novel \emph{neural control variates} method, (4) numerical experiments shows that NDSM learns better (in terms of inception score and FID) with fewer samples. }%(3) numerical experiments validate the benefits of adding structure. }

\section{Nonlinear denoising score matching}

\label{sec:NDSM}
Due to the nonlinearity of the dynamics in \eqref{eq:GM_nosing}, the exact transition probabilities are not known and therefore standard denoising score matching cannot be used. In this section we develop a novel {\bf nonlinear denoising score matching (NDSM)} method which can be used to train generative models with nonlinear forward noising dynamics.  The method leverages the fact that over a short timespan from $t_n$ to $t_{n+1}$ the transition probabilities for a nonlinear SDE  are approximately normal:
\begin{align}\label{eq:transition_prob_Gaussian}
p_{n+1}(dy_{n+1}|y_n)\sim N(\mu(y_n,t_n,\Delta t_n),\sigma^2(y_n,t_n,\Delta t_n)I)\,,
\end{align}
where $\Delta t_n=t_{n+1}-t_{n}$ (up to a final time $t_{n_f}=T$), for an appropriate mean $\mu$ and covariance $\sigma^2$.  Specifically, we use the Euler-Maruyama method which, for the SDE for $Y(s)$ in \eqref{eq:forward_backward_SDEs},  corresponds to
\begin{align}\label{eq:mu_sigma_def}
   \mu(y_n,t_n,\Delta t_n)=y_n-f(y_n,T-t_n)\Delta t_n\,,\,\,\,\sigma(y_n,t_n,\Delta t_n)= \tilde{\sigma}(T-t_n)\sqrt{\Delta t_n}\,.
\end{align}

In the following theorem we derive the  NDSM score matching objective  under the assumption that the transition probabilities have the form \eqref{eq:transition_prob_Gaussian}; in particular, our result applies to general nonlinear $f$ and not just a GMM.  One key feature that distinguishes the derivation of NDSM from standard DSM is that here we  add a specific mean-zero term to the objective in order to cancel a singularity that arises in the limit where the step-size $\Delta t\to 0$. Choosing $\Delta t$ to be small is required for accuracy of the simulation and so this innovation, which  dramatically reduces the variance of the objective, is necessary to  obtain a result that can be used in practice. 
\begin{theorem}[Nonlinear DSM]\label{thm:NDSM}
Let $Z_n\sim N(0,I)$, $n\in\mathbb{Z}^+$, $Y_0\sim \pi$ be independent and define 
\begin{align}\label{eq:Y_markov}
  Y_{n+1}=\mu(Y_n,t_n,\Delta t_n)+\sigma(Y_n,t_n,\Delta t_n) Z_{n+1}\,,\,\,\,n\geq0\,,  
\end{align}
so that $Y_n$ is a Markov process with  one-step transition probabilities  \eqref{eq:transition_prob_Gaussian}  and  initial distribution $\pi$. Denote the distribution of $Y_n$ by $\eta_n$ and its density by $\eta_n(y)$.    Let $N$ be a random timestep, valued in $\{1,...,n_f\}$ and  independent from  $Y_0$ and the $Z_n$'s. Then the score-matching optimization problem can be rewritten as follows:
 \begin{align}\label{eq:NDSM_loss}
&\argmin_\theta \frac{1}{2}\mathbb{E}\left[\|s_\theta(Y_{N},t_{N})-\nabla_y|_{Y_{N}}\log(\eta_{N}(y))\|^2\right]= \argmin_\theta\mathbb{E}\left[\mathcal{L}^{\mathrm{NDSM}}_{\theta,N}\right]\,,
\end{align}
where the NDSM loss is given by
\begin{align}\label{eq:NDSM_loss_def}
&\mathcal{L}^{\mathrm{NDSM}}_{\theta,N}\coloneqq \frac{1}{2}\|s_\theta(Y_{N},t_{N})\|^2+  \frac{1}{\sigma_{N-1}}Z_{N}\cdot(s_\theta(Y_{N},t_{N})-s_\theta(\mu_{N-1},t_{N}))\,,\\
&\mu_{n}\coloneqq \mu(Y_{n},t_{n},\Delta t_{n})\,,\,\,\,\sigma_{n}\coloneqq \sigma(Y_{n},t_{n},\Delta t_{n})\,.\notag
\end{align}
\end{theorem}
\begin{proof}
We will first consider the objective at a fixed time $t_n$. First expand the squared norm
 \begin{align}\label{eq:NDSM_loss_derivation1}
& \frac{1}{2}\Ex_{\eta_{n}}\left[\|s_\theta(y_n,t_n)-\nabla_y|_{y_n}\log(\eta_n(y))\|^2\right]\\
=&  \frac{1}{2}\Ex_{\eta_n}\left[\|s_\theta(y_n,t_n)\|^2\right]-\Ex_{\eta_n}[s_\theta(y_n,t_n)\cdot\nabla_{y}|_{y_n}\log(\eta_n(y))] +\frac{1}{2}\Ex_{\eta_n}\left[\|\nabla_y|_{y_n}\log(\eta_n(y))\|^2\right]\,,\notag
\end{align}
where $\nabla_y|_{y_n}$ denotes the gradient with respect to the variable $y$, evaluated at $y_n$. 

As in standard DSM, the  term on the right-hand side of \eqref{eq:NDSM_loss_derivation1} which does not depend on $\theta$ can be ignored as it does not impact the minimization over $\theta$. Focusing on the second term, which does depend on $\theta$ and also still contains the unknown density $\eta_n(y_n)$, we can write
\begin{align}
&\Ex_{\eta_n}[s_\theta(y_n,t_n)\cdot\nabla_{y}|_{y_n}\log(\eta_n(y))]
=\int s_\theta(y_n,t_n)\cdot\nabla_{y}|_{y_n}\eta_n(y)dy_n\\
=&\int s_\theta(y_n,t_n)\cdot\left(\nabla_{y}|_{y_n}\!\int\!\!...\!\!\int\!\!\int   p_{n}(y|y_{n-1})...p_{1}(y_1|y_0) \pi(dy_0)dy_1...dy_{n-1}\!\right)\!dy_n\notag\\
=&\int \!\!...\!\!\int\!\!\int  \! s_\theta(y_n,t_n)\cdot\nabla_y|_{y_n}p_{n}(y|y_{n-1})...p_{1}(y_1|y_0) \pi(dy_0)dy_1...dy_n\notag\\
=&\mathbb{E}\left[  s_\theta(Y_n,t_n)\cdot\nabla_y|_{Y_n}\log(p_{n}(y|Y_{n-1}))\right]\,.\notag
\end{align}
The  assumption \eqref{eq:Y_markov} implies
\begin{align}\label{eq:Y_Z_eq}
    Y_{n}=\mu_{n-1}+\sigma_{n-1} Z_{n}\,,
\end{align} 
where $\mu_{n-1}\coloneqq \mu(Y_{n-1},t_{n-1},\Delta t_{n-1})$, $\sigma_{n-1}\coloneqq \sigma(Y_{n-1},t_{n-1},\Delta t_{n-1})$, and therefore 
\begin{align}
  -\mathbb{E}\left[  s_\theta(Y_n,t_n)\cdot\nabla_y|_{Y_n}\log(p_{n}(y|Y_{n-1}))\right]  =&-\mathbb{E}\left[  s_\theta(Y_n,t_n)\cdot\nabla_y|_{Y_n}\left(  -\|y-\mu_{n-1}\|^2/(2\sigma_{n-1}^2)\right)\right]\notag\\
  =&\mathbb{E}\left[  s_\theta(Y_n,t_n)\cdot Z_n/\sigma_{n-1}\right]\,.\label{eq:NDSM_loss_derivation2}
\end{align}

To motivate the next step in the derivation we note that, when estimating \eqref{eq:NDSM_loss_derivation2} from samples, the $\sigma_{n-1}$ in the denominator can cause severe numerical problems, i.e., an extremely large variance, due to $\sigma_{n-1}$ becoming  small when $\Delta t_{n-1}$ is small; for instance, see \eqref{eq:mu_sigma_def}. To further clarify this issue  we perform the following formal calculations. Expand the score for small $\sigma_{n-1}$: 
\begin{align}
s^i_\theta(Y_n,t_n)=&s^i_\theta(\mu_{n-1}+\sigma_{n-1} Z_n,t_n)\\
  =&s^i_\theta(\mu_{n-1},t_n)+\nabla_y s^i_\theta(\mu_{n-1},t_n)\cdot \sigma_{n-1} Z_n+O(\sigma_{n-1}^2)\,.\notag
\end{align}
Therefore
\begin{align}\label{eq:bad_term_expansion}
 s_\theta(Y_n,t_n)\cdot Z_n/\sigma_{n-1}
=  s_\theta(\mu_{n-1},t_n)\cdot Z_n/\sigma_{n-1}+Z_n\cdot\nabla_y &s_\theta(\mu_{n-1},t_n)\cdot  Z_n   +O(\sigma_{n-1})\,.
\end{align}
Recalling   that $Z_n$ is independent of $Y_{n-1}$ we see that the expectation of the first term in \eqref{eq:bad_term_expansion} vanishes, however its variance diverges as $\sigma_{n-1}\to 0$. The other terms are well behaved as $\sigma_{n-1}\to 0$.  Therefore we have isolated the  troublesome behavior of \eqref{eq:NDSM_loss_derivation2} as originating from
\begin{align}\label{eq:W_def}
  W_{\theta,n}\coloneqq s_\theta(\mu_{n-1},t_n)\cdot Z_n/\sigma_{n-1}\,.
\end{align}
To obtain a numerically well-behaved objective we therefore subtract $W_{\theta,n}$  (which has expected value zero) from objective in \eqref{eq:NDSM_loss_derivation2} and then substitute the result into \eqref{eq:NDSM_loss_derivation1} to obtain
 \begin{align}\label{eq:NDSM_loss_derivation3}
& \frac{1}{2}\mathbb{E}\left[\|s_\theta(Y_n,t_n)-\nabla_y|_{Y_n}\log(\eta_n(y))\|^2\right]\\
=& \mathbb{E}\left[\frac{1}{2}\|s_\theta(Y_n,t_n)\|^2+  (s_\theta(Y_n,t_n)\cdot Z_n/\sigma_{n-1}-W_{\theta,n})\right]+\frac{1}{2}\mathbb{E}\left[\|\nabla_y|_{Y_n}\log(\eta_n(y))\|^2\right]\,.\notag
\end{align}
Taking the expectation over a random timestep $N$ (that is independent from the $Y_n$'s and $Z_n$'s) and substituting in \eqref{eq:W_def}   gives  
 \begin{align}\label{eq:NDSM_loss_derivation4}
& \frac{1}{2}\mathbb{E}\left[\|s_\theta(Y_{N},t_{N})-\nabla_y|_{Y_{N}}\log(\eta_{N}(y))\|^2\right]\\
=& \mathbb{E}\left[\frac{1}{2}\|s_\theta(Y_{N},t_{N})\|^2+  \frac{1}{\sigma_{N-1}} Z_{N}\cdot(s_\theta(Y_{N},t_{N})-s_\theta(\mu_{N-1},t_N))\right]+\frac{1}{2}\mathbb{E}\left[\|\nabla_y|_{Y_{N}}\log(\eta_{N}(y))\|^2\right]\,.\notag
\end{align}
Finally, minimizing  over $\theta$ and noting that the last term is independent of $\theta$ we arrive at \eqref{eq:NDSM_loss}.
\end{proof}
\begin{remark}
    A key practical difference between NDSM and standard DSM is that the sample trajectories must be simulated over a sequence of timesteps $t_n$; this is due to the nonlinear drift, which precludes us from having a general formula for the transition probabilities.  To increase computational efficiency we therefore find it advantageous to use multiple (random) timesteps from each sample trajectory in the loss, thereby reducing the number of trajectories that must be simulated for a given loss minibatch size.  One can also improve efficiency by computing in parallel a large batch of trajectories and sampling from them (with our without replacement), repeating this after every appropriate number of epochs.
\end{remark}

\subsection{Nonlinear DSM  with Neural Control-Variates}
In the derivation of \eqref{eq:NDSM_loss}, we used  the mean-zero term \eqref{eq:W_def} to prevent the variance of the NDSM objective from diverging as $\Delta t\to 0$. However, we can make further use of \eqref{eq:W_def} by introducing  additional learnable parameters, as in the method of control variates, see, e.g., Section 6.7.1 in \cite{rubinstein2009simulation}, to further reduce the variance. This leads us to propose the following {\bf nonlinear DSM  with control-variates (NDSM-CV)} method.
\begin{theorem}[Neural Control-Variates]\label{thm:NDSM_CV}
Let $Y_n$, $Z_n$, and $N$ be as in Theorem \ref{thm:NDSM}. For any continuous $\epsilon:[0,T]\to\mathbb{R}$, we have the following equivalence between the NDSM score-matching problem and the NDSM-CV problem:
\begin{align}\label{eq:NDSM_with_CV_optim}
  &\argmin_\theta \mathbb{E}[\mathcal{L}^{\mathrm{NDSM}}_{\theta,N}+\epsilon(t_N) W_{\theta,N}]   =\argmin_\theta \mathbb{E}[\mathcal{L}^{\mathrm{NDSM}}_{\theta,N}] \,,\\
  &W_{\theta,n}\coloneqq s_\theta(\mu_{n-1},t_n)\cdot Z_n/\sigma_{n-1}\,,\label{eq:W_def2}
\end{align}
where the NDSM loss was defined in \eqref{eq:NDSM_loss} and $\mu_n$, $\sigma_n$ were defined in \eqref{eq:NDSM_loss_def}.
\end{theorem}
\begin{proof}
Recalling that $N$ is independent from the $Z_n$'s and the $Y_n$'s we can compute
\begin{align}
  \mathbb{E}[\epsilon(t_N) W_{\theta,N}] =& \mathbb{E}_{n\sim N}\left[\mathbb{E}[\epsilon(t_n) W_{\theta,n}]\right]\\
  =&\mathbb{E}_{n\sim N}\left[\epsilon(t_n)\mathbb{E}[s_\theta(\mu_{n-1},t_n)\cdot Z_n/\sigma_{n-1}]\right]\notag\\
  =&\mathbb{E}_{n\sim N}\left[\epsilon(t_n)\mathbb{E}[s_\theta(\mu_{n-1},t_n)/\sigma_{n-1}]\cdot \mathbb{E}[Z_n]\right]=0\,.\notag
\end{align}
The last line follows from the independence of $Z_n$ and $Y_{n-1}$ and $\mathbb{E}[Z_n]=0$. This implies \eqref{eq:NDSM_with_CV_optim}.
\end{proof}
Fixing $\epsilon=0$ reduces NDSM-CV to NDSM,  while fixing $\epsilon=1$ completely reverses the cancellation of the singular term that was central to the derivation in Theorem \ref{thm:NDSM} and hence results in significant numerical issues; we demonstrate this in the examples in Figures \ref{fig:2d_square_true_samples}-\ref{fig:2d_squares_eps_trained} below.  We emphasize that the freedom to let $\epsilon$ depend on $t_N$ in \eqref{eq:NDSM_with_CV_optim} is due to $N$ being independent from $Z$ and $Y$.
In practice, we are not directly interested in estimating the loss but rather in its gradient with respect to $\theta$. Therefore we let $\epsilon$  be a neural network (NN) with parameters $\phi$ and train  it to reduce the MSE of the gradient:
\begin{align}\label{eq:epsilon_optim_NN}
\mathrm{argmin}_\phi \mathbb{E}\left[\|\nabla_\theta(\mathcal{L}^{\mathrm{NDSM}}_{\theta,N}+\epsilon_\phi(t_N) W_{\theta,N})-\mathbb{E}\left[\nabla_\theta \mathcal{L}^{\mathrm{NDSM}}_{\theta,N}\right]\|^2\right] \,.  
\end{align}
One can reduce this to a more standard control variate method by letting $\epsilon_\phi$ be a constant, in which case the minization \eqref{eq:epsilon_optim_NN} can easily be solved exactly.   We also note that if one wants to study the $\Delta t_n\to 0$ limit of \eqref{eq:NDSM_with_CV_optim} in the EM case and in the presence of finite $\epsilon$ then one should redefine $\epsilon$ to explicitly extract the $\Delta t_n^{1/2}$ factor necessary to cancel the scaling of $\sigma_{n-1}$ with $\Delta t_n$.  However, in practice, we implement this method with fixed finite $\Delta t_n$ and so the appropriate scaling of $\epsilon$ will be learned through the training process; in our tests it made little difference whether or not we explicitly enforced this $\Delta t_n$ scaling.

\begin{algorithm}
\caption{NDSM-CV Method}\label{alg:NDSM_CV_training}
\begin{algorithmic}[1]
\For{$\ell=1,...,N_{\mathrm{iterations}}$}
    \State Sample $\{y_i\}_{i=1}^B$ from $\pi$ and for each $i$, sample $k$ timesteps $n_{i,j}$, $j=1,...,k$.
    \State Simulate the forward noising dynamics (Algorithm \ref{alg:forward_noising}), starting from $y_i$ to obtain $y_{i,n}$
    \State Save the values of $y_{i,n_{i,j}}$ and the corresponding Gaussian noise samples $z_{n_{i,j}}$.
        \State $\theta \gets \theta-\gamma_1 \nabla_{\theta} \frac{1}{kB}\sum_{i=1}^B\sum_{j=1}^k \left(\mathcal{L}_{\theta,n_{i,j}}^{\mathrm{NDSM}}+\epsilon_\phi(t_{n_{i,j}})W^\theta_{n_{i,j}}\right)$ \,\,\,\, (see Eq.~\ref{eq:NDSM_loss_def} and Eq.~\ref{eq:W_def2})
        \If{$\ell$ is divisible by $N_{CV}$}
        \State $\phi \gets \phi-\gamma_2 \nabla_{\phi} \frac{1}{kB}\sum_{i,j}\|\nabla_\theta(\mathcal{L}_{\theta,n_{i,j}}^{\mathrm{NDSM}}+\epsilon_\phi(t_{n_{i,j}}) W_{\theta,{n_{i,j}}})-\frac{1}{kB}\sum_{i,j}\nabla_\theta \mathcal{L}_{\theta,n_{i,j}}^{\mathrm{NDSM}}\|^2$
    \EndIf
\EndFor
\end{algorithmic}
\end{algorithm}
%\begin{figure}
%\centering
% \vspace{-4.5mm}
%\includegraphics[width =0.7\textwidth]{figures/8_squares_asymmetric_control_variates_comparison/8squares_asymmetric_constant_GaussianMixture_ndsm_cv_50000_NNepsilon_Delta_t0.001_cv_update_interval20_epsilon_steps1_epsilon_plot.png}     \caption{$\epsilon_\phi(t)$ trained via \eqref{eq:epsilon_optim_NN} on the dataset  Figure \ref{fig:2d_square_true_samples}.}\label{fig:trained_epsilon}
%    \end{figure}
In Algorithm \ref{alg:NDSM_CV_training} we present the pseudocode for training \eqref{eq:NDSM_with_CV_optim} and \eqref{eq:epsilon_optim_NN} via SGD with learning rates $\gamma_1$ and $\gamma_2$. In practice, we use forward noising dynamics with GM drift (GM+NDSM-CV), as detailed in Section \ref{sec:GM_noising}, where the parameters of GM  are obtained from the data via an inexpensive preprocessing step. We note that updating $\epsilon_\phi$ according to \eqref{eq:epsilon_optim_NN} is expensive, due to the required per-sample gradient computations.  Therefore we only update after every $N_{CV}$ SGD updates of $\theta$; in practice we use $N_{CV}=20$, which we find gives good performance while having only a minor impact on computational cost.  At the cost of moderately higher variance, one can also simply fix $\epsilon=0$ (i.e., only canceling the singular term) in which case the $\phi$ update step is omitted in Algorithm \ref{alg:NDSM_CV_training}. In practice we find the optimal $\epsilon_\phi(t)$ to  be small  but  not identically zero.  
%In Figure \ref{fig:trained_epsilon} we plot the final $\epsilon_\phi(t)$ that was obtained after training  via \eqref{eq:epsilon_optim_NN} on the example from Figure \ref{fig:2d_square_true_samples} below. We note that the result is well-behaved and  remains small but is not identically zero.  

Finally, we  note $\Delta t_N$ can be chosen separately from $\Delta t_n$ for $n<N$; in practice, we observe that $\Delta t_N$ is the most important timestep, while $\Delta t_n$ for $n<N$ can often be chosen larger than $\Delta t_N$, thus reducing computational cost, without substantially impacting performance.  Strictly speaking, such a choice of timesteps requires a slight generalization of Theorems \ref{thm:NDSM} and \ref{thm:NDSM_CV} to allow the $t_n$  to depend on $N$, but it is straightforward to see that the theorems continue to hold in this case.  We emphasize that, in practice, the ability to work with a small $\Delta t_N$  relies heavily on the variance-reducing term we introduced in Theorem \ref{thm:NDSM} and built on in Theorem \ref{thm:NDSM_CV}.

\section{Low-dimensional examples}\label{sec:low_dim_examples}

\begin{figure}[h!]
\begin{subfigure}[t]{0.19\textwidth}
    \centering
    \includegraphics[width = \textwidth]{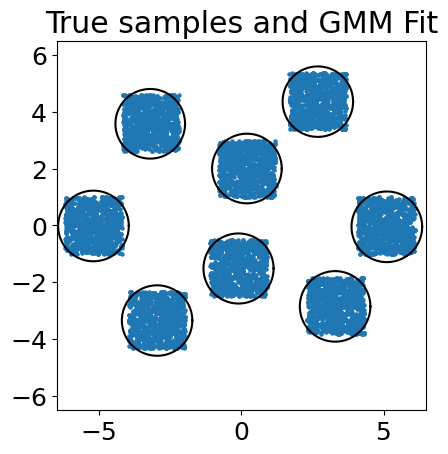} 
        \newsubcap{True samples and GMM modes.}\label{fig:2d_square_true_samples}
\end{subfigure}
\begin{subfigure}[t]{0.39\textwidth}
    \centering
    \includegraphics[width = .47\textwidth]{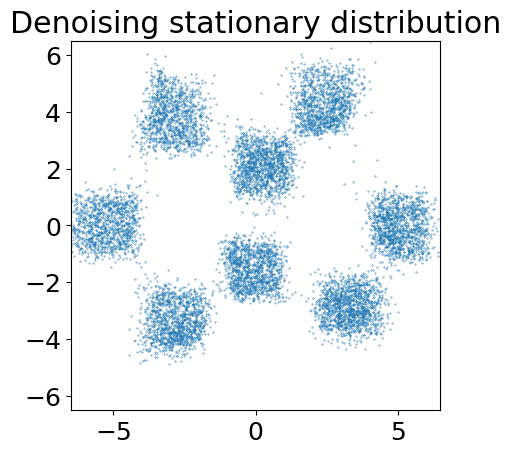} 
       \includegraphics[width = .47\textwidth]{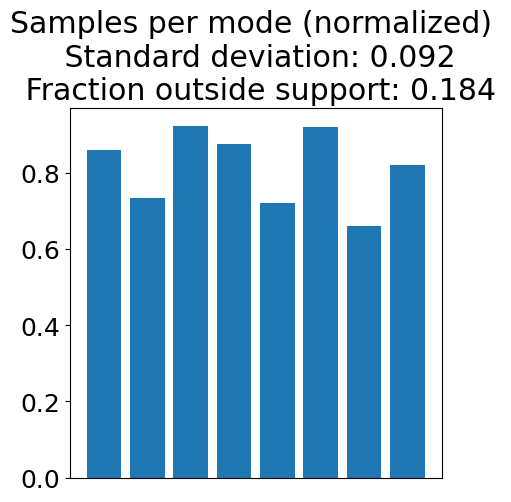} 
        \newsubcap{OU+DSM}\label{fig:2d_squares_OU_DSM}
\end{subfigure}
\begin{subfigure}[t]{0.4\textwidth}
    \centering
    \includegraphics[width = .47\textwidth]{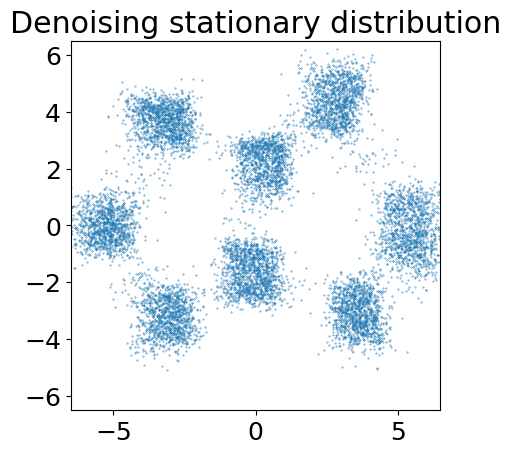} 
       \includegraphics[width = .47\textwidth]{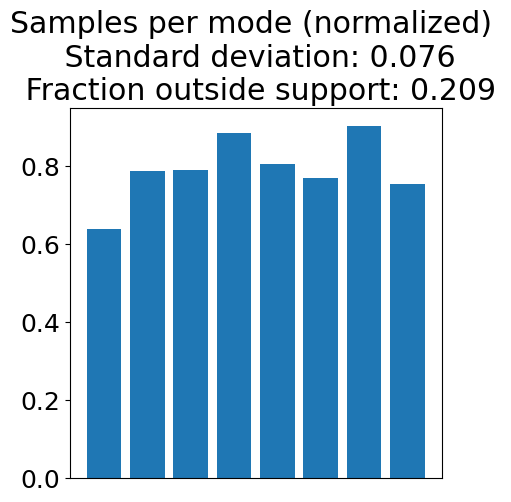}
        \newsubcap{GM+NDSM without CV (i.e., $\epsilon=1$)}\label{fig:2d_squares_eps1}
\end{subfigure}

\begin{subfigure}[t]{0.48\textwidth}
    \centering
    \includegraphics[width = .47\textwidth]{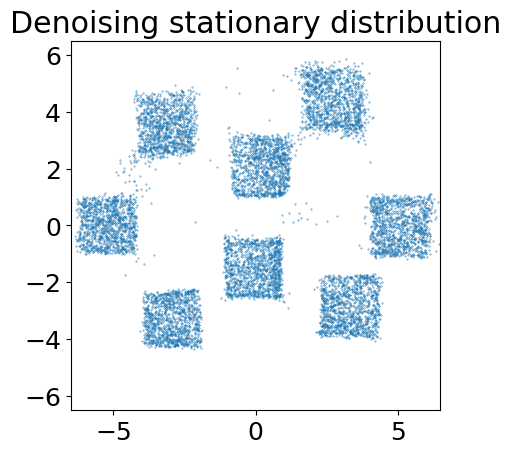} 
       \includegraphics[width = .47\textwidth]{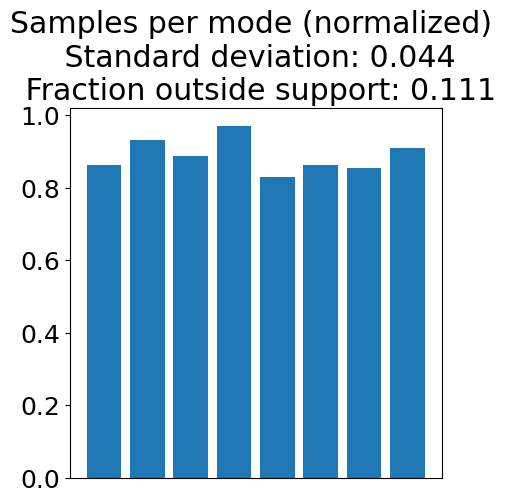} 
        \newsubcap{GM+NDSM-CV: $\epsilon=0$}\label{fig:2d_squares_eps0}
\end{subfigure}
\begin{subfigure}[t]{0.48\textwidth}
    \centering
    \includegraphics[width = .47\textwidth]{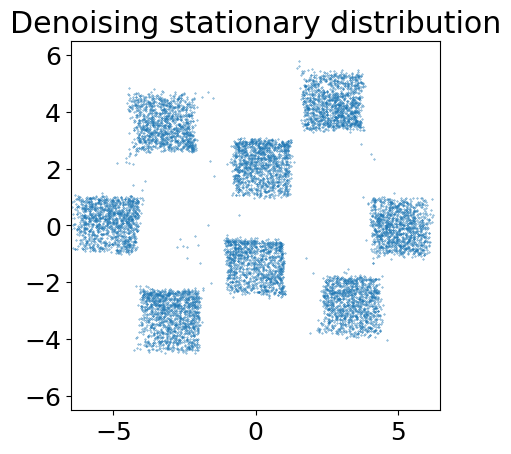} 
       \includegraphics[width = .47\textwidth]{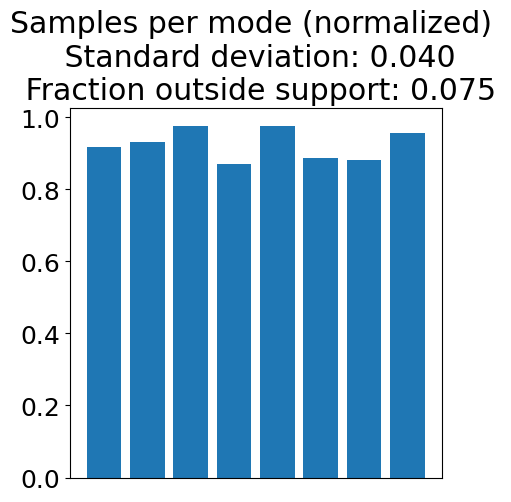} 
        \newsubcap{GM+NDSM-CV:   $\epsilon_\phi(t)$ trained via \eqref{eq:epsilon_optim_NN}}\label{fig:2d_squares_eps_trained}
\end{subfigure}
\caption{Comparison of sample quality and class balance on a 2D toy model example. Results for each method correspond to the median  value of the standard-deviation of  samples-per-mode    over 5 independent  runs.}\label{fig:8_squares}
\end{figure}

In this section we present  low-dimensional multi-modal distribution examples which illustrate  key aspects of the NDSM-CV method with Gaussian mixture forward dynamics (GM+NDSM-CV); these examples are motivated by clustering in latent space after auto-encoding, as in, e.g., \cite{ding2019deciphering,li2023searching}; see also the example in Section \ref{sec:MNIST_latent} below, where we employ it on MNIST. Here the GMM is learned from 10000 samples, as described in Section \ref{sec:GM_noising}; this makes a negligible contribution to the overall computational cost. We will compare these results with those of linear (OU) forward dynamics and denoising score matching (OU+DSM).  These examples use a score model with 7 fully connected hidden layers, each with 32 nodes, and GELU activations.  The control-variate method uses $\epsilon_\phi$ that has 3 fully connected hidden layers, each with 10 nodes, and ReLU activations. Both are trained using the Adam optimizer with a learning rate of $10^{-3}$ on 10000 training samples. The score models were trained  for 50000 SGD steps and $\epsilon_\phi$ was updated every 20 iterations (where applicable). The losses all use a minibatch size of 250.  In the case of DSM, each minibatch consists of 250 samples from the data distribution, evolved under the linear dynamics up to a random time.  In the case of NDSM-CV, each minibatch consists of 50 samples evolved under the  nonlinear dynamics, with each trajectory samples at 5 random times along the trajectory.   The denoising dynamics used 1000 timesteps in all cases.

In Figures \ref{fig:2d_square_true_samples}-\ref{fig:2d_squares_eps_trained}
 we compare standard DSM, Figure \ref{fig:2d_squares_OU_DSM}, with the NDSM-CV methods, both with fixed $\epsilon$ and with $\epsilon_\phi(t)$ trained via \eqref{eq:epsilon_optim_NN}. The NDSM-CV methods   use GM noising dynamics, simulated over 50 timesteps, with  $\Delta t_N=0.001$.    NDSM-CV with $\epsilon=0$ corresponds to a complete cancellation of the term singular that becomes singular as $\Delta t_N\to 0$, while $\epsilon=1$ corresponds to undoing this cancellation; see the discussion before Theorem \ref{thm:NDSM} as well as the proof. This fact explains the discrepancy between Figures \ref{fig:2d_squares_eps1} and \ref{fig:2d_squares_eps0}, with the former ($\epsilon=1$) showing degraded performance while the performance of the latter ($\epsilon=0$) is greatly improved due to the variance reduction; it particular, it significantly outperforms  DSM. %The effect  of variance reduction is even more apparent when training for fewer SGD steps; see Figure \ref{fig:2d_squares_10000} below. 
 The best performance in this example is obtained when $\epsilon_\phi(t)$ is trained via \eqref{eq:epsilon_optim_NN}, see Figure \ref{fig:2d_squares_eps_trained}, both in terms of matching the support of the target distribution as well as having better balance between the modes, though the improvement over the simpler $\epsilon=0$ case is relatively minor compared to the effect of  non-adaptive variance reduction (i.e., compared to moving from $\epsilon=1$ to $\epsilon=0$).

We emphasize that the GM+NDSM methods all have improved class balance over OU+DSM in Figure \ref{fig:8_squares} due to the use of the GMM prior, which incorporates the   mode weights learned during the preprocessing step, as described in Section \ref{sec:GM_noising}. The difference in class balance  between Figures 3(c), 3(d), and 3(e) is not due to a difference in mixture weights; they all use the same GMM fitting procedure, which does provide accurate mode weights in this case.  The difference in performance between 3(c) and  3(d)~-~3(e) is that the training objective for 3(c) has much higher variance (exploding as $\Delta t_N\to 0$), resulting in a poorly trained  score model. In contrast,  3(d)-3(e) use the key variance-reducing corrections from our Theorems \ref{thm:NDSM} and \ref{thm:NDSM_CV} respectively.

\begin{figure}[t]
\begin{subfigure}[t]{0.19\textwidth}
    \centering
    \includegraphics[width = \textwidth]{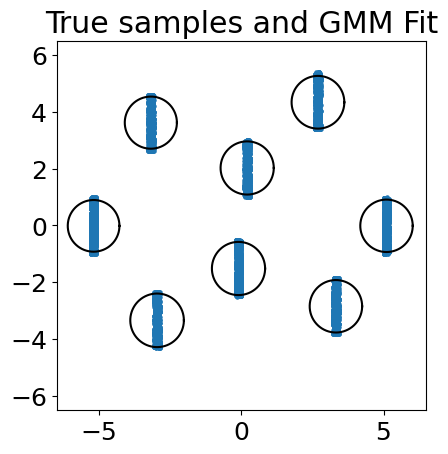} 
        \newsubcap{True samples}\label{fig:2d_thin_true_samples}
\end{subfigure}
    \begin{subfigure}[t]{0.39\textwidth}
    \centering
    \includegraphics[width = .47\textwidth]{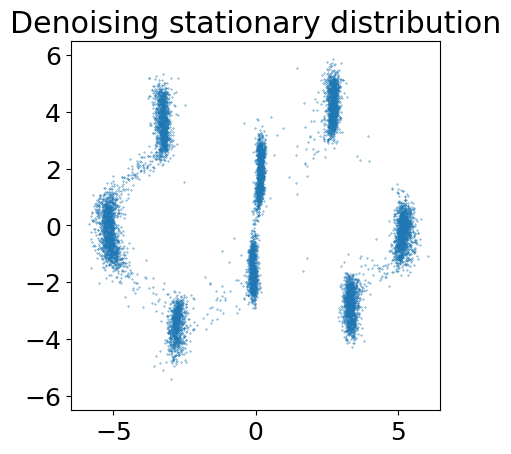} 
       \includegraphics[width = .47\textwidth]{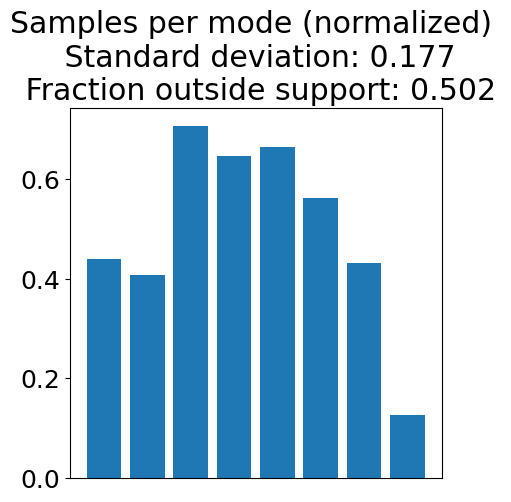} 
        \newsubcap{OU+DSM}\label{fig:2d_thin_DSM}
\end{subfigure} 
\begin{subfigure}[t]{0.4\textwidth}
    \centering
    \includegraphics[width = .47\textwidth]{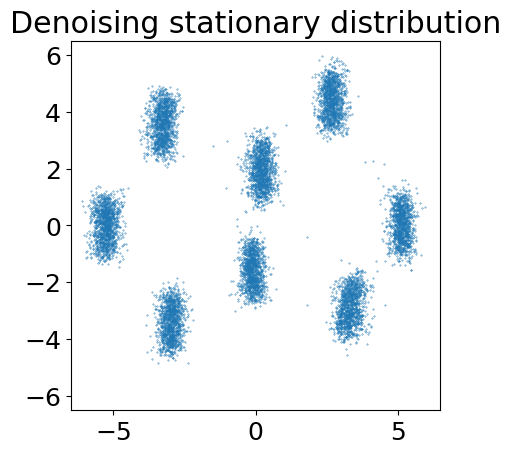} 
       \includegraphics[width = .47\textwidth]{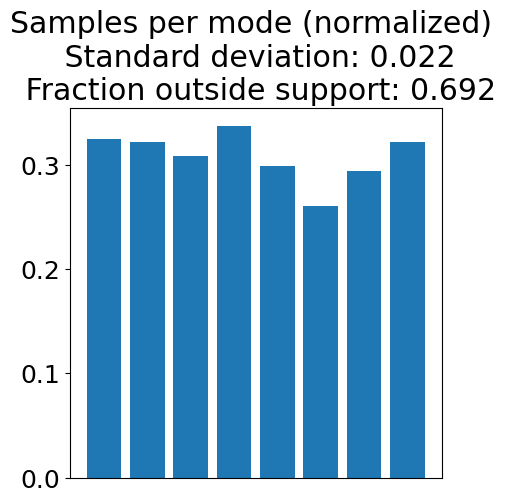} 
        \newsubcap{GM+NDSM-CV: $\Delta t_N=0.05$}
\end{subfigure}

\begin{subfigure}[t]{0.48\textwidth}
    \centering
    \includegraphics[width = .47\textwidth]{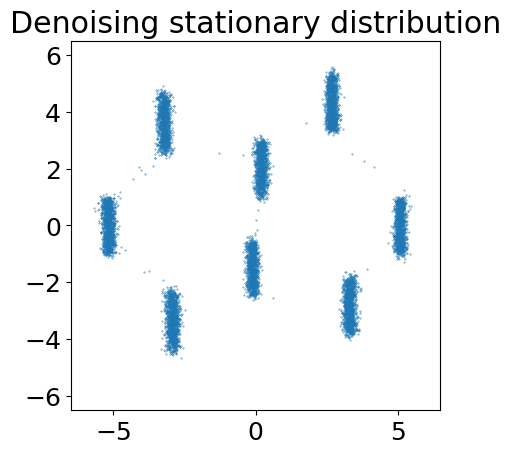} 
       \includegraphics[width = .47\textwidth]{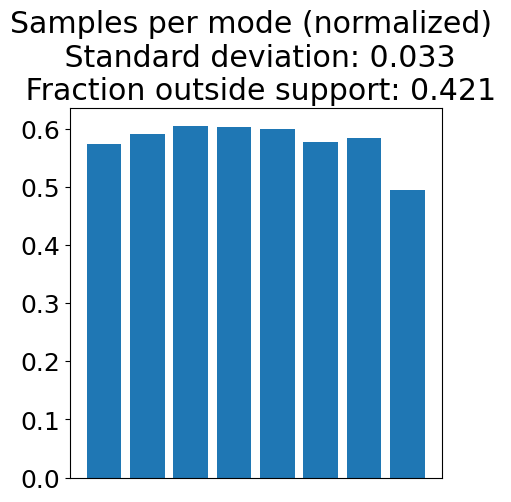} 
        \newsubcap{GM+NDSM-CV: $\Delta t_N=0.01$}
\end{subfigure}
\begin{subfigure}[t]{0.48\textwidth}
    \centering
    \includegraphics[width = .47\textwidth]{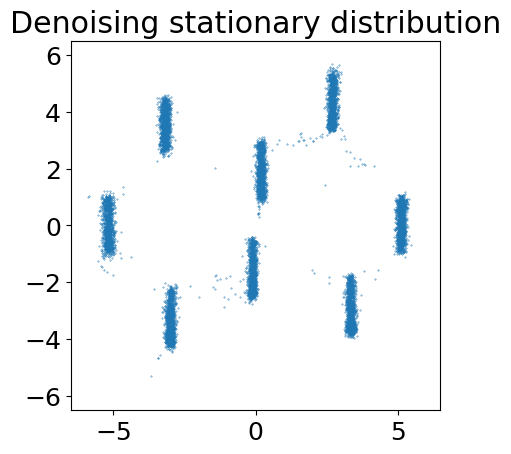} 
       \includegraphics[width = .47\textwidth]{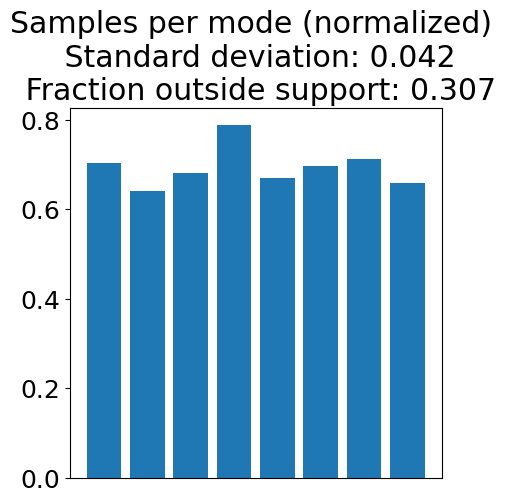} 
        \newsubcap{GM+NDSM-CV: $\Delta t_N=0.005$}\label{fig:2d_thin_NDSM_005}
\end{subfigure}
\caption{Comparison of sample quality and class balance on a 2D toy model example. Results for each method correspond to the median  value of the standard-deviation of  samples-per-mode    over 5 independent  runs.}
\end{figure}

In Figure \ref{fig:2d_thin_true_samples}-\ref{fig:2d_thin_NDSM_005} we show the effect of the  perturbation stepsize $\Delta t_N$ for timesteps  that enter into the NDSM loss \eqref{eq:NDSM_loss_def}, while fixing $\Delta t_n$ for $n<N$ to be $0.00998$.     We compare with DSM, Figure \ref{fig:2d_thin_DSM}, which does not have an analogous $\Delta t_N$ parameter.  We observe that a smaller $\Delta t_N$ is more capable of learning  distributions with (approximate) lower dimensional support. The GM dynamics also lead to samples that are more evenly distributed among the modes, similar to what was seen in Figures \ref{fig:2d_squares_OU_DSM}-\ref{fig:2d_squares_eps_trained}.

\section{High-dimensional image examples}\label{sec:image_examples}

Finally, we demonstrate the enhanced performance of our proposed GM+NDSM-CV model in learning high-dimensional \textit{structured} distributions, compared to the benchmark OU+DSM. Specifically, we evaluate their performance using the following two datasets:
    
\textbf{(a) MNIST:} A collection of 60,000 handwritten digits stored as $28\times 28$ grayscale images \citep{lecun1998gradient}. This dataset inherently represents a multi-modal distribution, with each digit class forming at least one mode.

\textbf{(b) MNIST in Latent Space:} A 4-dimensional latent space representation of the standard MNIST dataset.

\textbf{(c) Approx.-$C_2$-MNIST:} This dataset is constructed by randomly rotating MNIST digits by $180^{\circ}$ with a probability of 1/2 and resizing to half-size, creating a distribution that is approximately---\textit{but not exactly}---invariant under the discrete rotation group $C_2$. For a visual illustration of the image samples, refer to Figure~\ref{fig:mnist_nearly_c2_real}. It is important to note that the smaller digits are always upside-down, whereas the larger digits (those that are not transformed) remain upright.
\begin{figure}[h]
\centering
\includegraphics[width =0.6\textwidth]{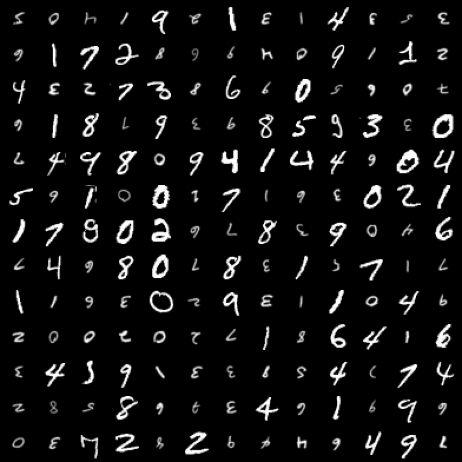}
\caption{Random samples from the Approx.-$C_2$-MNIST dataset. Note that smaller digits are always upside-down, whereas the larger digits (those that are not transformed) remain upright.}\label{fig:mnist_nearly_c2_real}
\end{figure}

\subsection{Implementation Details}\label{sec:MNIST_implementation}
The MNIST examples use the U-net \citep{ronneberger2015u} architecture as the backbone of the score network $s_\theta(y, t)$. More specifically, the encoder part of the score model comprises four blocks with decreasing spatial resolution, each containing a $3\times 3$  convolution layer, group normalization layer, and a ``swish'' activation function \citep{ramachandran2017searching}. Time information is incorporated via Gaussian random features \citep{tancik2020fourier} and propagated through fully connected layers in each encoder block. The decoder, defined similarly with increasing spatial resolution, includes skip connections from the encoding to the decoding path.

The MNIST in latent space example uses an autoencoder where the encoder consists of three $3\times 3$ convolutional layers, $1\to 8\to 16\to 32$, all with stride of $2$ and with  BatchNorm2d layers between the convolutional layers. The result is  flattened and followed by two fully connected layers $288\to 128\to 4$, with $4$ being the latent space dimension; all layers use  ReLU activations. The decoder follows the same structure in reverse order, with an additional sigmoid layer at the end. The autoencoder was trained for $50$ epochs using the Adam optimizer with learning rate of $10^{-3}$, weight decay of $10^{-5}$, and with a batch size of 256. This example uses the same score model (with the exception of the initial data dimension being $4$ rather than $2$), control-variate net, minibatch size, and optimizer  as in the $2$-dimensional examples from Section \ref{sec:low_dim_examples}. The model was trained for $100$ epochs.

For the benchmark OU+DSM, we consider mainly the Variance Preserving (VP) SDE \citep{sohl2015deep,ho2020denoising,song2020score} as the forward diffusion process:
\begin{align}
    dY(t) = -\frac{1}{2}\beta(t)Y(t)dt + \sqrt{\beta(t)}dW(t),
\end{align}
where $\beta(t)$ is a linear function on $[0, T]$, with $\beta(0) = 0.1$ and $\beta(T) = 20$. The terminal time is set to $T=1$.  For NDSM-CV, we use Langevin dynamics with the preprocessed Gaussian Mixture as the stationary distribution for the forward diffusion process, with the terminal time set to $T=2$. 

Unless stated otherwise above, the models were trained using the Adam optimizer with a batch size of 64 for 100 epochs. The computations for examples (a) and (c) were on a Linux machine equipped with a GeForce RTX 3090 GPU and example (b) was done using the    LEAP2 HPC resources at Texas State University.

For the Approx.-$C_2$-MNIST dataset, we also consider group-equivariant score models, achieved by symmetrizing a standard score model under the $C_2$ group \citep{lu2024structure, hoogeboom2022equivariant}; such models have shown superior performance over their non-equivariant counterparts in learning distributions that are \textit{exactly} group invariant.

\subsection{MNIST}\label{sec:MNIST_example}
\begin{table} 
\parbox{.4\linewidth}{\centering
\caption{IS and FID on MNIST}
  \label{tab:result_mnist}
  \centering
  \begin{tabular}{lcc}
    \toprule
    Model     & IS$\uparrow$     & FID$\downarrow$\\
    \midrule
    OU+DSM & 6.76  & 143.3     \\
    \midrule
    Our model; $\epsilon=0$ & 8.82 & 37.4\\
    Our model; $\epsilon_{\phi}(t)$ trained & \textbf{8.93} & \textbf{36.1}\\
    \bottomrule 
  \end{tabular}
}
\hfill 
\parbox{.4\linewidth}{ 
  \centering
    \caption{Low data MNIST ($N = 14000$)}
  \begin{tabular}{lcc}
    \toprule
    Model     & IS$\uparrow$     & FID$\downarrow$\\
    \midrule
    OU+DSM & 5.17 & 470.92    \\
    \midrule
    Our model; $\epsilon_{\phi}(t)$ trained & \textbf{6.89} & \textbf{190.63}\\
    \bottomrule 
  \end{tabular}
}
\end{table}

A mixture of ten Gaussians fitted to MNIST digits is used as the stationary (prior) distribution for GM-NDSM-CV. For computational efficiency, the covariance matrix of each component is a constant but potentially distinct multiple of the identity matrix. We use a small timestep, $\Delta t = 10^{-3}$, to ensure accurate simulation of the forward nonlinear SDE; we use this same timestep for the denoising dynamics.

The top row of Figure~\ref{fig:mnist_all} displays random samples generated by different models, and the bottom row shows the class distribution of these samples. The benchmark OU+DSM (Figure~\ref{fig:mnist_ou_dsm}) exhibits significant mode imbalance, predominantly generating ``easy digits'' such as 1 (over 30\%) and 7. Our models without or with neural control variate, shown in Figure~\ref{fig:mnist_ndsm_fixed} and Figure~\ref{fig:mnist_ndsm_optimized} respectively, address this issue by producing more evenly distributed samples across all classes with consistent quality. 
\begin{figure}[h!]
\centering
\begin{subfigure}[t]{0.29\textwidth}
    \centering
    \includegraphics[width = \textwidth]{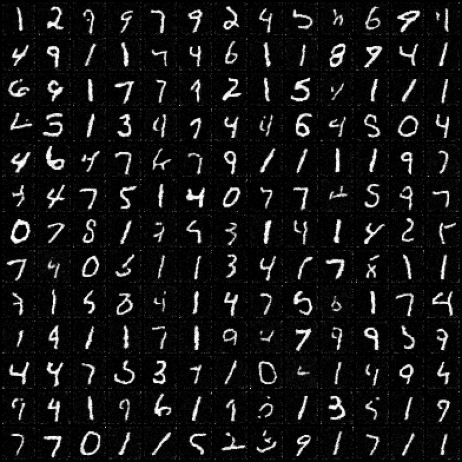} 
\end{subfigure}~~
\begin{subfigure}[t]{0.29\textwidth}
    \centering
    \includegraphics[width = \textwidth]{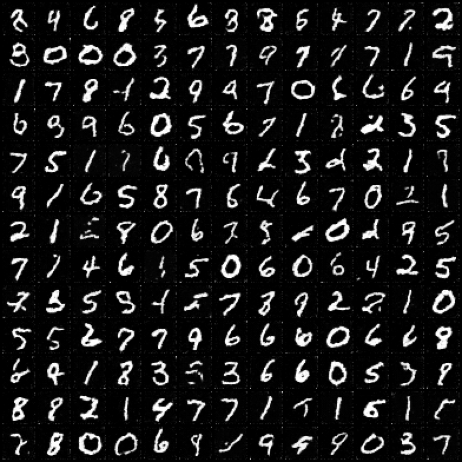}        
\end{subfigure}~~
\begin{subfigure}[t]{0.29\textwidth}
    \centering
    \includegraphics[width = \textwidth]{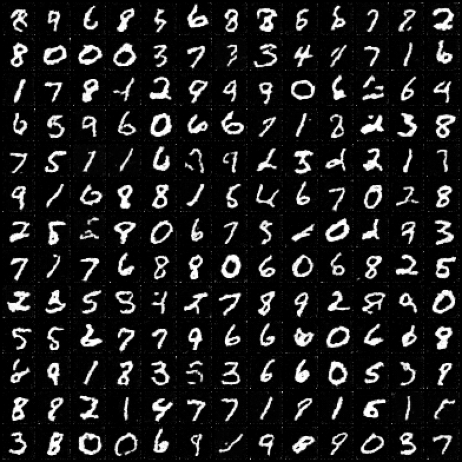}
\end{subfigure}\\
\begin{subfigure}[t]{0.29\textwidth}
    \centering
    \includegraphics[width = \textwidth]{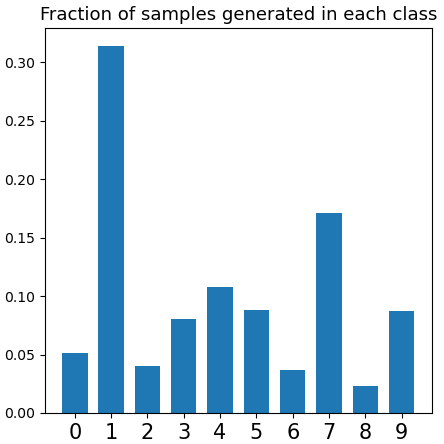} 
    \subcaption{OU+DSM}
    \label{fig:mnist_ou_dsm}
\end{subfigure}
\begin{subfigure}[t]{0.29\textwidth}
    \centering
    \includegraphics[width = \textwidth]{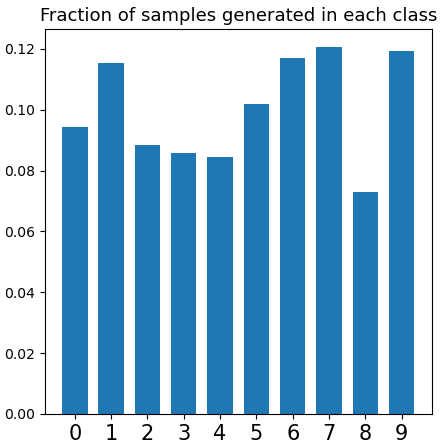}
    \subcaption{GM+NDSM-CV:  $\epsilon=0$}
    \label{fig:mnist_ndsm_fixed}
\end{subfigure}
\begin{subfigure}[t]{0.29\textwidth}
    \centering
    \includegraphics[width = \textwidth]{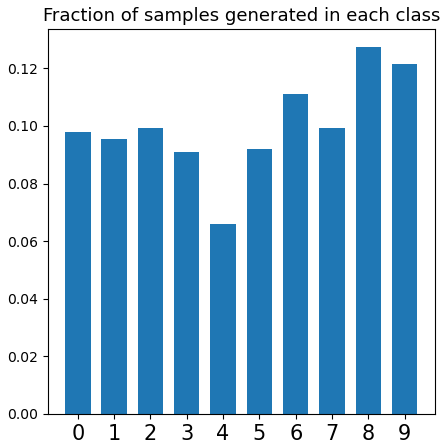}
    \subcaption{GM+NDSM-CV: $\epsilon_\phi(t)$ trained}
    \label{fig:mnist_ndsm_optimized}
\end{subfigure}
\caption{Top row: random samples generated by different models on the MNIST dataset. Bottom row: the fraction of samples generated in each digit class from 0 to 9.}\label{fig:mnist_all}
\end{figure}
\begin{figure}[h!]
\begin{subfigure}{0.48\textwidth}
    \centering
    \includegraphics[width = .48\textwidth,trim = {45 45 45 45},clip]{figures/low_data_pics/nlsm_6000.png}
    \includegraphics[width = .48\textwidth,trim = {45 45 45 45},clip]{figures/low_data_pics/ou_dsm_6000.png}
    \subcaption{$N=6000$: NDSM  (left),   DSM (right)}
    \end{subfigure}
%     \begin{subfigure}{0.32\textwidth}
%     \includegraphics[width = \textwidth,trim = {45 45 45 45},clip]{figures/low_data_pics/nlsm_10000.png}
%     \subcaption{GM+NDSM-CV, N = 10000}
%     \end{subfigure}
%\begin{subfigure}{0.24\textwidth}
 %   \centering
  %  \includegraphics[width = \textwidth,trim = {45 45 45 45},clip]{figures/low_data_pics/ou_dsm_6000.png}
   % \subcaption{OU+DSM, N = 6000}
    %\end{subfigure}
    % \begin{subfigure}{0.33\textwidth}
    % \includegraphics[width = \textwidth,trim = {45 45 45 45},clip]{figures/low_data_pics/ou_dsm_10000.png}
    % \subcaption{OU+DSM, N = 10000}
    % \end{subfigure}
        \begin{subfigure}{0.48\textwidth}
    \includegraphics[width = .48\textwidth,trim = {45 45 45 45},clip]{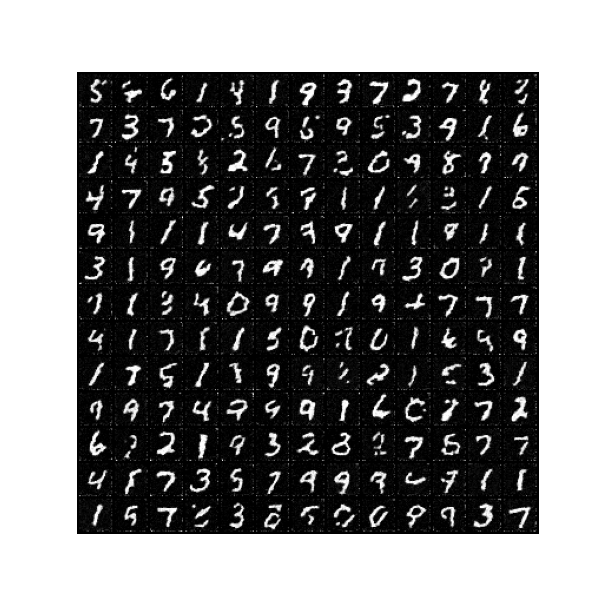}
     \includegraphics[width = .48\textwidth,trim = {45 45 45 45},clip]{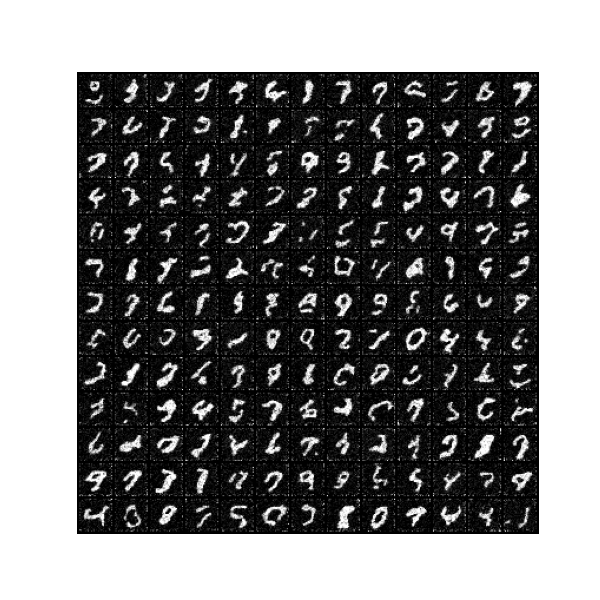}
     \subcaption{$N = 14000$: NDSM (left), DSM (right)}
\end{subfigure} 
\caption{MNIST in the low data regime ($N$ training samples), comparing OU+DSM with our new GM+NDSM-CV method. GM+NDSM-CV can learn well with less data. }\label{fig:MNIST_less_data}
\end{figure}

 Table~\ref{tab:result_mnist} presents the \textit{inception score} (IS, higher is better) \citep{salimans2016improved} and \textit{Fr\'{e}chet inception distance} (FID, lower is better) \citep{NIPS2017_8a1d6947}, evaluated using a pre-trained ResNet18 classifier on MNIST. These metrics affirm that our models significantly surpass the benchmark OU+DSM, with the neural control variate model achieving the best results, corroborating the visual evidence in Figure~\ref{fig:mnist_all}. In addition, in Figure \ref{fig:MNIST_less_data} we demonstrate that the structure encoded in the GM drift yields an SGM that learns with fewer training samples.
We note that for the MNIST examples we have used a smaller NN (see Section \ref{sec:MNIST_implementation}) than employed by some other studies so as to reduce the computational cost and focus on the effects of nonlinear dynamics and corresponding score matching method.

\subsection{Approx.-$C_2$-MNIST}
\begin{figure}[h]
\centering
\begin{subfigure}[t]{0.43\textwidth}
    \centering
    \includegraphics[width = .67\textwidth]{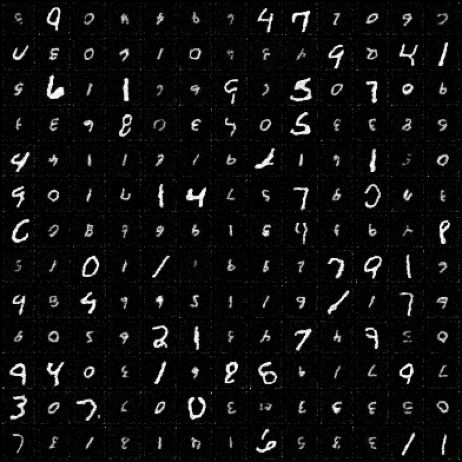} 
    \includegraphics[width = .30\textwidth]{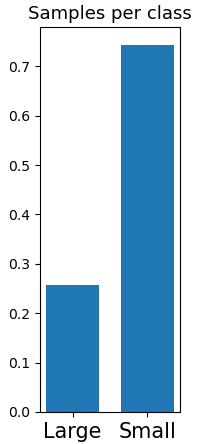}
        \subcaption{OU+DSM}
        \label{fig:mnist_nearly_c2_ou_dsm}
\end{subfigure}~~~~
\begin{subfigure}[t]{0.43\textwidth}
    \centering
    \includegraphics[width = .67\textwidth]{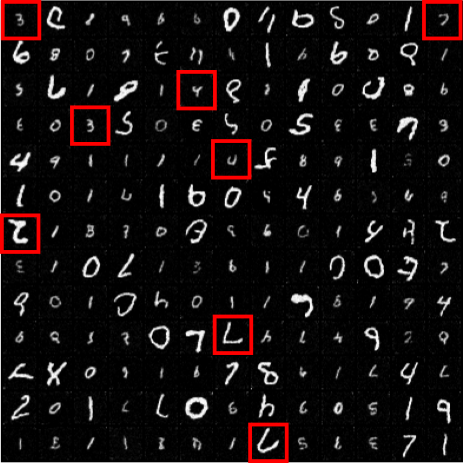} 
    \includegraphics[width = .30\textwidth]{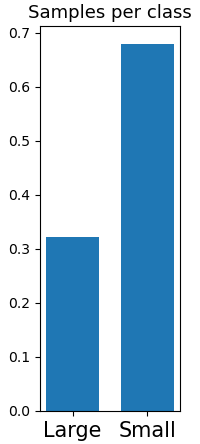}
        \subcaption{OU+DSM, equivariant score model}
        \label{fig:mnist_nearly_c2_ou_dsm_equiv}
\end{subfigure}\\
\begin{subfigure}[t]{0.43\textwidth}
    \centering
    \includegraphics[width = .67\textwidth]{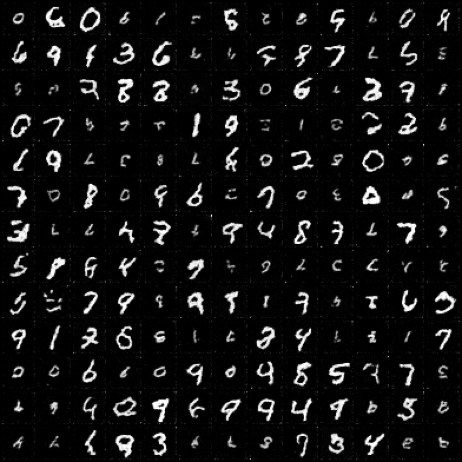} 
    \includegraphics[width = .30\textwidth]{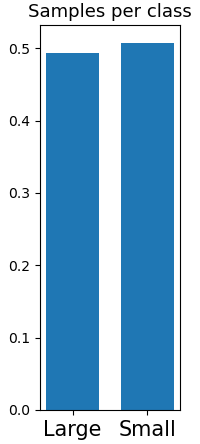}
        \subcaption{GM+NDSM-CV, $\epsilon=0$}
        \label{fig:mnist_nearly_c2_ndsm}
\end{subfigure}~~~~
\begin{subfigure}[t]{0.43\textwidth}
    \centering
    \includegraphics[width = .67\textwidth]{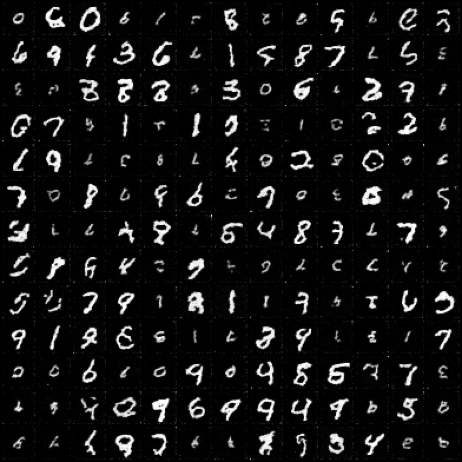} 
    \includegraphics[width = .30\textwidth]{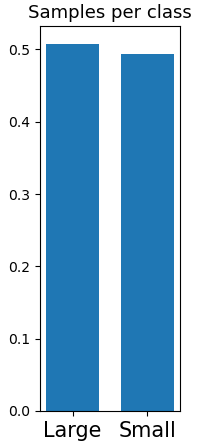}
        \subcaption{GM+NDSM-CV, $\epsilon_{\phi}(t)$ trained}
        \label{fig:mnist_nearly_c2_ndsm_optimized}
\end{subfigure}
\caption{Approx.-$C_2$-MNIST. Problematic ``fake samples'', such as large-but-upside-down and small-but-upright digits, are highlighted in panel (b) when using an equivariant score model.}\label{fig:mnist_nearly_c2all}
\vspace{-10pt}
\end{figure}
Given the approximate $C_2$-symmetry, we fit two Gaussians with shared covariance matrices to the dataset as a preprocessing step for GM-NDSM-CV. Figure~\ref{fig:mnist_nearly_c2all} showcases random samples from various models. Notably, the benchmark OU+DSM, both without (Figure~\ref{fig:mnist_nearly_c2_ou_dsm}) and with (Figure~\ref{fig:mnist_nearly_c2_ou_dsm_equiv})   an equivariant score model, consistently exhibits mode imbalance, predominantly generating the ``easy mode'' of small digits. Additionally, the OU+DSM with the equivariant score model introduces another issue by producing ``fake samples'' such as large-but-upside-down and small-but-upright digits (highlighted in red in Figure~\ref{fig:mnist_nearly_c2_ou_dsm_equiv}), indicating that the model erroneously learns from the $C_2$-symmetrized version of the underlying distribution. For comparison, see Figure~\ref{fig:mnist_nearly_c2_real} for true samples from this dataset. These issues are effectively addressed by our models (Figures~\ref{fig:mnist_nearly_c2_ndsm} and \ref{fig:mnist_nearly_c2_ndsm_optimized}), which leverage a flexible design of the stationary distribution and a tailored nonlinear diffusion for the data.

\subsection{MNIST in Latent Space}\label{sec:MNIST_latent}
After compressing the MNIST dataset using the autoencoder described in Section \ref{sec:MNIST_implementation}, we fit a Gaussian mixture model to 1000 samples from the latent space representation and compare the OU+DSM and GM+NDSM-CV methods in Table \ref{tab:result_mnist_latent}. Here we trained both methods until convergence and report the average performance over the last 5 epochs.  The denoising dynamics used 1000 timesteps in both cases.

We again find the performance of our method to be a substantial improvement over OU+DSM.  We emphasize that our method in this setting are computationally  less expensive than in Section \ref{sec:MNIST_example}; the significantly reduced dimensionality lowers the cost associated with simulating the nonlinear SDE, allows for larger timesteps to be used  and allows for the use of neural networks with many fewer parameters, all while providing similar performance in terms of FID  and significantly improved performance in terms of IS as compared to our method in Section \ref{sec:MNIST_example}. Our method here used $\epsilon=0$; the results with $\epsilon_\phi(t)$ trained are similar, though the increased computational cost is not justified.
\begin{table}[h!]
\caption{IS and FID on MNIST:\\
4-Dimensional Latent Space Representation}
  \label{tab:result_mnist_latent}
  \centering
  \begin{tabular}{lcccc}   
     & \multicolumn{2}{c}{Trained to Convergence}       &  \multicolumn{2}{c}{Equal Training Time} 
        \\
         \toprule
    Model   & IS$\uparrow$     & FID$\downarrow$   & IS$\uparrow$     & FID$\downarrow$\\
    \midrule
    OU+DSM & 9.14
  & 67.4 &8.99&78.3
     \\
    \midrule
    GM+NDSM-CV & {\bf 9.52} & {\bf 42.9} & {\bf 9.42}& {\bf 45.1}
\\    
    \bottomrule 
  \end{tabular}
\end{table}
The performance gain of GM+NDSM-CV over OU+DSM persists even when accounting for the computational cost of both methods. This is shown  in Figure \ref{fig:performance_vs_time} as well as in the two rightmost columns of Table \ref{tab:result_mnist_latent}, where we report the average performance  over the last 5 epochs shown in the figure for each method. Specifically, we plot the FID and IS versus training time, allowing for OU+DSM to run for more epochs in order to equalize the training time (on the same set of resources) of the two methods.  Here we see that the GM+NDSM-CV method converges much faster than  OU+DSM. Thus there is a wider performance gap when both methods are allocated equal training time  than when both methods are trained  to convergence.

\begin{figure}[b]
    \centering
    \includegraphics[width = .47\textwidth]{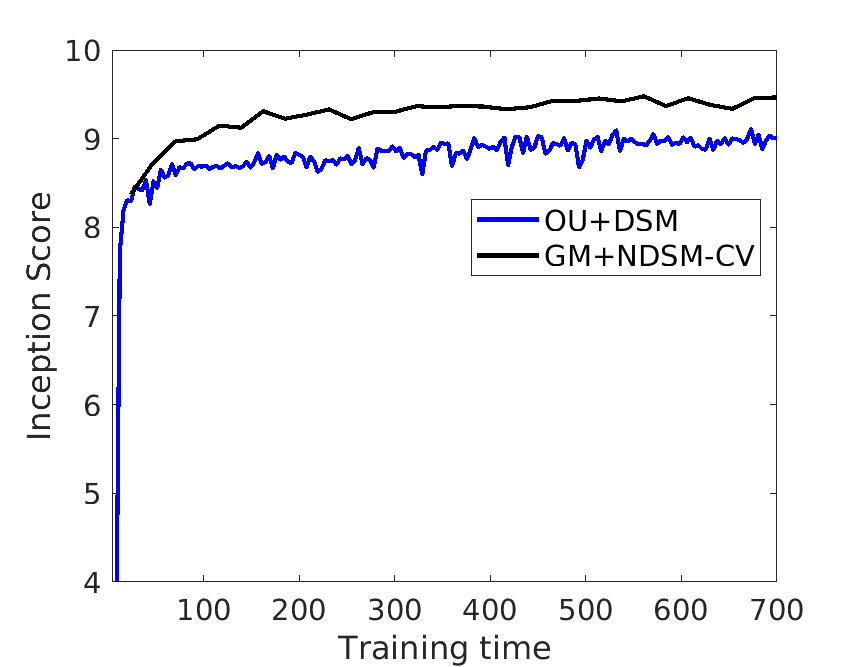} 
       \includegraphics[width = .47\textwidth]{figures/performance_vs_time/FID_time.png} 
        \caption{Performance versus computation time for OU+DSM and GM+NDSM-CV methods. }\label{fig:performance_vs_time}
\end{figure}

\section{Conclusions}

% \MK{1. Mention likelihood-free inference as a potential application?}

% \MK{2. In abstract/intro/conclusions include sentences that align with aspects (highlighted a couple next) of the special issue: "This special issue focuses on generative AI methodologies that meet the unique demands of science and engineering applications, such as \textbf{limited data} [MK: stress the performance of our methods in low data regimes in abstract/intro/etc ] availability; complex physics that are governed by nonlinear, non-stationary, and multi-scale phenomena; dynamic environments with rapidly changing conditions; and \textbf{domain knowledge} [MK: e.g. "approximate" symmetries, multi-modality] provided in the form of physical laws and constraints. Furthermore, generative AI techniques for high-consequence decision-making in science and engineering require rigorous validation and verification, robustness, interpretability, and uncertainty quantification to maintain the integrity of the overall process. This special issue aims to capture the current state and articulate future directions of this rapidly evolving field, fostering a dialogue among researchers in mathematics, computational science, and engineering."}

The NDSM-CV method is a general purpose method for training generative models that allows for the use of nonlinear noising dynamics, thereby incorporating appropriate information on the structure of the data into the dynamics and prior distribution.  We have demonstrated that generative models trained using the proposed NDSM-CV method can attain significant improvement in the quality of generative models and allow them to be trained with substantially smaller data sets.  The necessity of simulating the nonlinear SDE still adds  additional computational cost, thus we use several techniques, such as reusing sample paths but with different randomly sampled timesteps, to increase the efficiency. We note the relation to the contemporaneous local-DSM approach of \cite{singhals}, which also proposed the use of nonlinear noising dynamics for training SGMs. Our work is differentiated in several respects: 1) we introduced a new variance-reduced  NDSM-loss in Theorem \ref{thm:NDSM}, achieved through identifying and canceling a   high-variance mean-zero term, 2) we proposed a novel neural control variates method to further lower the variance   in Theorem \ref{thm:NDSM_CV}, 3) we utilized priors obtained by fitting (a subset of) the data via an inexpensive preprocessing step.  We find the latter to be especially effective when applied to latent space representations of high-dimensional data.

For future work, we foresee using customized nonlinear noising processes for improving generative models for inference applications.  Likelihood-free or simulation-based inference \cite{cranmer2020frontier} is a popular application of generative models \cite{baptista2024bayesian}, and score-based models have been adapted to perform conditional sampling \cite{batzolis2021conditional}. Bayesian posteriors with multiple modes are especially challenging to sample \cite{shaw2007efficient}, and SGMs with nonlinear diffusion processes provides a way to address their multimodal structure.

\section*{Acknowledgements}
M. Katsoulakis, B. Zhang, L. Rey-Bellet  are  partially funded by AFOSR grant FA9550-21-1-0354. M.K. and L. R.-B. are  partially funded by NSF DMS-2307115. M.K. is partially funded  by NSF TRIPODS CISE-1934846. W. Zhu is partially supported by NSF grants DMS-2052525, DMS-2140982, and DMS-2244976, as well as AFOSR grant FA9550-25-1-0079. This work was enabled in part by Texas State University scientific computational resources provided by the LEAP2 High Performance Computing service.

% \bibliographystyle{unsrt}
% \vfil
%\eject

%\bibliography{bibliotheque.bib}
\bibliography{arxiv_version.bbl}

%%%%%%%%%%%%%%%%%%%%%%%%%%%%%%%%%%%%%%%%%%%%%%%%%%%%%%%%%%%%

\end{document}